\documentclass{article}

\usepackage[accepted]{icml2023}

\usepackage[fleqn]{nccmath}

\usepackage[utf8]{inputenc} 
\usepackage[T1]{fontenc}    
\usepackage{hyperref}       
\usepackage{url}            
\usepackage{booktabs}       
\usepackage{amsfonts}       
\usepackage{nicefrac}       
\usepackage{microtype}      
\usepackage{xcolor}         
\usepackage{subcaption}
\usepackage{tabularray}
\usepackage{adjustbox}
\usepackage{hyperref}
\usepackage{url}
\usepackage{xspace}
\usepackage{algorithm2e}
\RestyleAlgo{ruled}
\usepackage{amsmath}
\usepackage{amsthm}
\usepackage{wrapfig}
\usepackage{authblk}
\usepackage{float}

\usepackage{colortbl}
\newcommand{\method}{\textsc{Counterpol}\xspace}

\newcommand{\rtarget}{R_{\text{target}}}
\newcommand{\pithetacf}{\pi_{\theta_\text{cf}}}
\newcommand{\Jpitheta}{J_{\pi_{\theta}}}

\newcommand{\xhdr}[1]{\vspace{0em}\noindent{{\bf #1.}}}

\newcommand{\ie}{\textit{i.e., }}
\newcommand{\eg}{\textit{e.g., }}
\newcommand{\std}[1]{\scriptsize{$\pm$#1}}

\newcolumntype{a}{>{\columncolor{Gray}}c}
\newcolumntype{b}{>{\columncolor{white}}c}
\usepackage{booktabs} 
\usepackage{color, colortbl}
\definecolor{LightCyan}{rgb}{0.88,1,1}

\newcommand{\norm}[1]{\left\lVert#1\right\rVert}

\newtheorem{theorem}{Theorem}[section]

\DeclareMathOperator*{\argmax}{arg\,max}
\DeclareMathOperator*{\argmin}{arg\,min}

\usepackage[textsize=tiny]{todonotes}

\icmltitlerunning{Counterfactual Explanation Policies in RL}

\begin{document}

\twocolumn[
\icmltitle{Counterfactual Explanation Policies in RL}



\icmlsetsymbol{equal}{*}


\begin{icmlauthorlist}
\icmlauthor{Shripad V. Deshmukh}{aa}
\icmlauthor{Srivatsan R.}{aa,bb}
\icmlauthor{Supriti Vijay}{aa,cc}
\icmlauthor{Jayakumar Subramanian}{aa}
\icmlauthor{Chirag Agarwal}{dd}
\end{icmlauthorlist}

\icmlaffiliation{aa}{Media and Data Science Research, Adobe}
\icmlaffiliation{bb}{Indian Institute of Technology, Madras}
\icmlaffiliation{cc}{Manipal Institute Of Technology}
\icmlaffiliation{dd}{Harvard University. Work done while at Adobe}

\icmlcorrespondingauthor{Chirag Agarwal}{chiragagarwall12@gmail.com}

\icmlkeywords{Machine Learning, ICML}

\vskip 0.3in
]



\printAffiliationsAndNotice{}  

\begin{abstract}
    As Reinforcement Learning (RL) agents are increasingly employed in diverse decision-making problems using reward preferences, it becomes important to ensure that policies learned by these frameworks in mapping observations to a probability distribution of the possible actions are explainable. However, there is little to no work in the systematic understanding of these complex policies in a contrastive manner, \ie what minimal changes to the policy would improve/worsen its performance to a desired level. In this work, we present \method, the first framework to analyze RL policies using counterfactual explanations in the form of minimal changes to the policy that lead to the desired outcome. We do so by incorporating counterfactuals in supervised learning in RL with the target outcome regulated using desired return. We establish a theoretical connection between \method and widely used trust region-based policy optimization methods in RL. Extensive empirical analysis shows the efficacy of \method in generating explanations for (un)learning skills while keeping close to the original policy. Our results on five different RL environments with diverse state and action spaces demonstrate the utility of counterfactual explanations, paving the way for new frontiers in designing and developing counterfactual policies.   
\end{abstract}
\section{Introduction}
\label{sec:intro}
Reinforcement learning (RL) has been used successfully to train autonomous agents capable of achieving better than human-level performance in simulated environments like Go~\citep{go_rl} and a suite of Atari games~\citep{atari_rl}. They are increasingly finding new applications across computational analysis~\citep{matrix},  marketing~\citep{marketing},  education~\citep{tutor}, and biomedical research~\citep{protein} and. Recent breakthroughs in large-language models (LLMs)~\citep{gpt4} are primarily attributed to key RL components that have improved the generative capability of state-of-the-art LLMs. With RL frameworks being deployed at scale as well as performing autonomously, it becomes imperative to incorporate explainability in them, resulting in increased user trust in autonomous decision-making. Explaining the decisions of black-box RL agents for a given environment state is non-trivial as it not only involves explaining the final agent action but also includes complex decision-making and planning behind the output action.

A myriad of RL explainability methods with various attribution techniques has recently been proposed~\citep{greydanus2018visualizing,deshmukhexplaining,iyer2018transparency,puri2019explain}. In particular, they focus on identifying input states and past experiences (trajectories) that led the RL agent to learn complex behaviors. While these methods output important input state features (agent's observation) and trajectories, they fail to explain the minimal change in the trained policy leading to a desired outcome or (un)learning of a specific skill. Intuitively, this requires generating \textit{counterfactuals} for a given desired outcome (\ie identifying \textit{what} and \textit{how much} to change a given RL policy to obtain a target return for its current state). While some previous works have explored causal reinforcement learning~\citep{causal_rl}, there is little to no research on systematically explaining the mechanism of the complex policies learned by a given RL agent using counterfactual explanations.

\xhdr{Present Work} We propose \method, a framework for counterfactual analysis of RL policies. In our framework, we generate explanations by asking the question: ``What least change to the current policy would improve or worsen it to a new policy with a specified target return?'' To estimate such counterfactual policies, we present an objective that aims to obtain a new policy with an average performance equal to that of a specified return while limiting its modifications with respect to the given policy. The generated policies provide direct insights into how a policy can be modified to achieve better results as well as what to avoid in order not to deteriorate the performance. Further, we theoretically prove the connection between popular trust region-based optimization methods in RL with \method, bringing a new perspective of looking at RL optimization using a prominent explainability tool. Formally, the \method learns minimal changes in the current policy without changing its general behavior. To optimize counterfactual explanation policies, we specify a novel objective function that can be solved using basic on-policy Monte Carlo policy gradients. In our experiments across diverse RL environments, we show how our algorithm reliably achieves counterfactual for any the set target return for a given policy.

\xhdr{Our Contributions} We present our contributions as follows: 1) We formalize the problem of counterfactual explanation policy for explaining RL policies. 2) We propose \method, an explanatory framework for generating contrastive explanations for RL policies that identify \textit{minimal} changes in the current policy, which would lead to improving/worsening the performance of a given policy. 3) We derive a theoretical equivalence between the \method objective with the widely used trust region-based policy gradient methods. 4) We demonstrate the flexibility of \method through empirical evaluations of explanations generated for five OpenAI gym environments. Qualitative and quantitative results show that \method successfully generates a counterfactual policy for (un)learning skills while keeping close to the original policy.
\section{Related Work}
\label{sec:related_work}
This work lies at the intersection of counterfactual explanations, explainability in RL, and proximal gradients. Next, we discuss the related work for each of these topics.

\xhdr{Counterfactual Explanations} Several techniques have been recently proposed to generate counterfactual explanations for providing recourses to individuals negatively impacted by complex black-box model decisions~\citep{wachter2017counterfactual,Ustun2019ActionableRI,van2019interpretable,mahajan2019preserving,karimi2019model}. In particular, these techniques aim to provide model explanations in the form of minimal changes to an input instance that changes the original model. Broadly, these methods are categorized based on access to the predictive model (white- vs. black-box), enforce sparsity (more interpretable explanations), and whether the counterfactuals belong to the original data manifold~\cite{verma2020counterfactual}. To this end, ~\citet{wachter2017counterfactual} is one of the most widely used methods that use a gradient-based method to obtain counterfactual explanations for models using a distance-based penalty. In this work, we extend the counterfactual explanations to RL policies. 

\xhdr{Explainability in RL} Explainable RL (XRL)~\citep{puiutta2020explainable} methods are a sub-field of explainable artificial intelligence (XAI) that aims to interpret and explain the complex behaviour of RL agents. Recent works in XRL are broadly categorized into i) gradient-based methods~\citep{greydanus2018visualizing}, which analyzed Atari RL agents that use raw visual input to make their decision and identify salient information in the image that the policy uses to select an action, ii) attribution-based methods~\citep{deshmukhexplaining, puri2019explain, iyer2018transparency}, which explain an agent’s behaviour by attributing its actions to global or local observation data, and iii) surrogate-based methods, which distill complex RL policy into simpler models such as decision trees~\citep{coppens2019distilling} or to human understandable decision language~\citep{verma2018programmatically}. While all existing XRL methods generate explanations using state features or trajectories, none explore counterfactual explanations. To the best of our knowledge, for the first time, we explore counterfactual explanation policies for contrastive analysis of policies in terms of what minimal change to the policy would get us to desired output returns.

\xhdr{Proximal Gradient Methods in Optimization} General idea behind proximal gradient methods is to solve an optimization problem $\min_x f(x)$ in the neighborhood of certain $x_{k-1}$ at $k^{th}$ optimization iteration by modifying the objective as $\min_x f(x) + \lambda \cdot \norm{x - x_{k-1}}_p$~\citep{parikh2014proximal,NIPS2014_d554f7bb}. This objective modification is termed as proximal operator. In machine learning, these operators find great significance in regularization, constraint-based optimization, convex structuring of objectives, etc. They have several applications in a multitude of fields ranging from risk-aware forecasting systems, and mathematical reasoning to reinforcement learning~\citep{9341559,asadi2023faster, pmlr-v162-ding22b, hirano2022policy, pmlr-v100-khan20a, lu2023dynamic}. One major advantage of proximal gradient methods is in their ability to handle noisy and incomplete data, which is common in real-world applications. Subsequently, in RL, they have been referred by state-of-the-art algorithms like PPO~\citep{ppo} to develop ways to achieve monotonic gains. In this light and in contrast, we incorporate a proximal operator to ensure counterfactual RL policies are closer to the original input policy. 
\section{Preliminaries}
\label{sec:prelims}

\xhdr{Supervised Learning} Consider a standard classification setup with a given set of training examples, 
\[\mathbb{X} = \{(\mathbf{x}^{(i)}, y^{(i)})~|~\mathbf{x}^{(i)} \in \mathbb{R}^d, y^{(i)} \in C\},\] 
where $i \in \{1, 2, \dots, N\}$, $C = \{1, 2, \dots, K\}\}$, $\mathbf{x}^{(i)}$ is a $d$-dimension vector, $N$ is the number of training examples, and $C$ denotes set of $K$ classes. Let a classifier $f$ be trained on $\mathbb{X}$ to predict the correct label for any unseen input $\textbf{x}$.

\xhdr{Counterfactual Explanations}
XAI literature~\citep{cfxai_survey,cfml_survey} defines counterfactual explanations as a technique to analyze  ``what if'' scenarios for model predictions, \eg a counterfactual explanation for a loan application model rejecting an individual's loan application could be ``if you increased your income by \$500, your application would have been accepted''.  Note that in the present work we work in a non-causal setup, where following the previous works~\citep{wachter2017counterfactual}, we define a counterfactual explanation for a given model prediction of input $\mathbf{x}$ as the \textit{minimal variation in the features of $\mathbf{x}$} that changes the prediction of the underlying model from $y$ to $y_{\text{target}}$. 

Formally, in for the aforementioned supervised learning setup,  \textit{counterfactual explanation} of model prediction $\hat{y}_0 = f(\mathbf{x}_0)$ conditioned on a target class $y_{\text{target}} \in C$ is given by,
\begin{equation}
    \mathbf{x}_{\text{cf}} = \argmin_{\mathbf{x}}[ \text{CE}(f(\mathbf{x}),~y_{\text{target}}) + k\cdot \norm{\mathbf{x}_{\text{0}} - \mathbf{x}}_2^2 ],
    \label{eqn:cf_sl}
\end{equation}
where cross-entropy (CE) loss guides explanation for achieving $y_{\text{target}}$ by modifiying $\mathbf{x}_0$ to $\mathbf{x}_{\text{cf}}$ and the mean-squared distance  between $\mathbf{x}_{\text{cf}}$ and $\mathbf{x}_0$ ensures the proximity between the two, regulated via coefficient $k$.

\xhdr{Reinforcement Learning} Consider a finite horizon Markov Decision Process (MDP)~\citep{mdp} defined as $\mathcal{M} = (\mathcal{S}, \mathcal{A}, P, R, d_0, \gamma)$, where $\mathcal{S}$ denotes the state-space,  $\mathcal{A}$ denotes action-space, $P: (\mathcal{S} \times \mathcal{A} \times \mathcal{S}) \xrightarrow{} [0, 1]$ denotes state transition function, $R: (\mathcal{S} \times \mathcal{A} \times \mathcal{S}) \xrightarrow{} \mathbb{R}$ denotes the reward function, $d_0: \mathcal{S} \xrightarrow[]{} [0,1]$ represents distribution over starting states, and $\gamma \in (0, 1]$ denotes the discount factor. Let $\pi: \mathcal{S} \times \mathcal{A} \xrightarrow[]{} [0, 1]$ denote the learnt agent policy. Then, we measure the performance $J_{\pi}$ of the policy in terms of expected return as follows:
\begin{equation}
    J_{\pi} = \mathbb{E}_{(s_0, a_0, s_1, a_1, \dots, s_T)}\Big[\sum_{t=0}^{T-1}{\gamma^{t}r(s_t, a_t, s_{t+1})}\Big],
    \label{eqn:j_pi}
\end{equation}
where $s_0 \sim d_0$, $a_t \sim \pi(a_t | s_t)$, and $T$ denotes the episode terminating time-step.
\section{Counterfactual Explanation Policies}
\label{sec:pol_cf}

Reinforcement learning is a powerful decision-making tool that provides a systematic abstraction for defining environment dynamics with reward preferences and also provides the algorithms to come up with policies for maximizing returns in that environment. In the context of RL agents, we pose the counterfactual question as follows -- given a policy $\pi$ performing at level $J_{\pi}$ in an MDP $\mathcal{M}$, what infinitesimal change in $\pi$ would lead us to a target return of $\rtarget$? Here, we refer to such variation of $\pi$ as \textit{Counterfactual Explanation Policy} conditioned on $\rtarget$. In posing the counterfactual question in terms of target returns, we aim to get contrastive insights into what minimal changes to the current policy can result in improving/worsening its performance to a desired level. To estimate counterfactual explanation policy $\pi_{\text{cf}}$, we first define the return penalty for achieving the target return as:
\begin{equation}
    \label{eqn:return_penalty}
    \begin{aligned}
    L_{\text{ret}} = {} & \norm{J_{\pi} - \rtarget}_p, \text{where $\norm{\cdot}_p$ is $\ell_p$-norm.}
    \end{aligned}
\end{equation}
In order to ensure that the counterfactual policy is similar to the original policy, we limit the changes while achieving the target return. We use the KL-divergence loss to measure the distance between the original policy $\pi_0$ and a given policy $\pi$, \ie, 
\begin{equation}
    \label{eqn:loss_kl}
    L_{\text{KL}} = D_{\text{KL}}(\pi_0 || \pi)
\end{equation}
Next, we can find the counterfactual policy $\pi_{\text{cf}}$ by minimizing the following objective:
\begin{equation}
    \label{eqn:loss_cf_policies}
    \begin{aligned}
        \pi_{\text{cf}} =  {} & \argmin_{\pi}~~ L_{\text{ret}} + k \cdot L_{\text{KL}} \\ = {} & \argmin_{\pi}~~ \norm{J_{\pi} - \rtarget}_p + k \cdot D_{\text{KL}}(\pi_0 || \pi),
    \end{aligned}
\end{equation}
where $k$ is a regularization coefficient that ensures that the output counterfactual policy is close to $\pi_0$.
\subsection{Counterfactual Explanation Policy Optimization}
\label{sec:cf_pol_optimization}
In practice, RL policies are represented using function approximators~\citep{fun_approx_1, fun_approx_2}, \eg a policy is represented using a neural network in deep RL. 
Let $\pi_{\theta}$ denote the policy approximated using a neural network function with parameters $\theta$. We can rewrite Eqn.~\ref{eqn:loss_cf_policies} as:
\begin{equation}
    \label{eqn:loss_cf_policies_theta}
    \begin{aligned}
        \pithetacf =  {} & \argmin_{\theta}~~ L_{\text{ret}, \theta} + k \cdot L_{\text{KL}, \theta} \\ = {} &  \argmin_{\theta}~~ \norm{\Jpitheta - \rtarget}_p + k \cdot D_{\text{KL}}(\pi_{\theta_0} || \pi_\theta)
    \end{aligned}
\end{equation}
The above optimization requires the gradients of the counterfactual objectives \textit{w.r.t.} $\theta$, \ie $\nabla_{\theta} (L_{\text{ret}, \theta} + k \cdot L_{\text{KL}, \theta})$, which can be written as $\nabla_{\theta}L_{\text{ret}, \theta} + k \cdot \nabla_{\theta} D_{\text{KL}}(\pi_{\theta_0} || \pi_\theta)$.

Computing the gradients of Eqn.~\ref{eqn:loss_cf_policies_theta} is not directly feasible using automated differentiation provided in standard deep learning libraries~\citep{pytorch,tensorflow}. Therefore, we detail the derivation of how to estimate the gradients of the two loss terms in Eqn.~\ref{eqn:loss_cf_policies_theta}.

\xhdr{Estimating $\mathbf{\nabla_{\theta}L_{\text{ret}, \theta}}$}
Without loss of generality, we calculate the gradients of the return penalty using $\ell_{1}$-norm for the return loss objective $L_{\text{ret}, \theta}$, \ie
\begin{equation}
    \label{eqn:loss_ret_grad_1}
    \begin{aligned}
    \nabla_{\theta}L_{\text{ret}, \theta} =  \nabla_{\theta}|\Jpitheta - \rtarget| = {} & \text{sgn}(\Jpitheta - \rtarget) \cdot \nabla_{\theta}\Jpitheta,
    \end{aligned}
\end{equation}
where $\nabla_{\theta}\Jpitheta$ becomes the standard policy gradient~\citep{polgrad}.

We approximate $\Jpitheta$ with the average return gathered through the on-policy rollout of $\pi_\theta$ in the given environment. Formally, for $N$ episodes sampled using $\pi_\theta$, \ie $\{\tau_i\}_{i=1}^N \sim \pi_{\theta}$ we have:
\begin{equation}
    \label{eqn:jpitheta_approx}
    \Jpitheta \approx \frac{1}{N}\sum_{i=1}^{N}\sum_{t=0}^{T_i-1}{\gamma^{t}r(s_t, a_t, s_{t+1})}
\end{equation}
Further, using the policy gradient theorem~\citep{polgrad,sutton_book}, we can write $\nabla_\theta \Jpitheta  = \mathbb{E}_{\pi_\theta}[Q_{\pi_\theta}(s, a)\nabla_\theta \text{log}(\pi_\theta(s, a))]$. We approximate $Q_{\pi_\theta}(s, a)$ term by the return obtained starting from state $s$ and action $a~{\sim}~\pi_\theta(\cdot|s)$ for a sampled episode, \ie $Q_{\pi_\theta}(s, a)~{\approx}~\sum_{j=t}^{T-1}~{\gamma^{j-t}r(s_j, a_j, s_{j+1} | s_t = s, a_t = a)}$. Therefore,
\begin{fleqn}
  \begin{equation}
    \label{eqn:polgrad_approx}
    \begin{aligned}[b]
        & \nabla_\theta \Jpitheta  \approx \frac{1}{\sum_{i=1}^{N}{T_i}} \cdot\\
        & \sum_{i=1}^{N}{\sum_{t=0}^{T_i-1}[\sum_{j=t}^{T_i-1}{\gamma^{j-t}r(s_j, a_j, s_{j+1} | s_t, a_t)] \cdot \nabla_\theta \text{log}(\pi_\theta(s_t, a_t))}}
    \end{aligned}
  \end{equation}
\end{fleqn}

Note that the $\Jpitheta$ computation can be made more sample efficient by off-policy estimation~\cite{munos2016safe, jiang2016doubly} of $\Jpitheta$ on rollouts performed using $\pi_{\theta_0}$~\cite{a2c_1,trpo,ppo}. However, for the sake of simplicity of exposition and implementation, we use On-policy Monte Carlo Policy Gradients, which can later be modified to their sample efficient versions depending on the environmental requirements.

\xhdr{Estimating $\mathbf{\nabla_\theta D_{\text{KL}}(\pi_{\theta_0}||\pi_\theta)}$} As discussed in Eqn.~\ref{eqn:loss_kl}, we calculate KL-divergence between two policy distributions as the distance penalty for counterfactual policies. As estimating the exact KL-divergence between two policies (\ie $\mathbb{E}_s[D_{\text{KL}}(\pi_0(\cdot|s)||\pi(\cdot|s))]$) is computationally expensive for large state-action spaces, we approximate it by averaging the KL-divergence calculated over states visited during sampling of $N$ episodes. Using states in $\{\tau_i\}_{i=1}^N \sim \pi_{\theta}$, we write:
\begin{equation}
\label{eqn:dkl_approx}
    D_{\text{KL}}(\pi_{\theta_0}||\pi_\theta) \approx \frac{\sum_{i=1}^{N} \sum_{t=0}^{T_i-1}D_{\text{KL}}(\pi_{\theta_0}(\cdot|s_t)||\pi_{\theta}(\cdot|s_t))}{\sum_{i=1}^{N}{T_i}}
\end{equation}
and finally approximate $\nabla_\theta D_{\text{KL}}(\pi_{\theta_0}||\pi_\theta)$ as: 
\begin{equation}
\label{eqn:grad_dkl_approx}
    \nabla_\theta D_{\text{KL}}(\pi_{\theta_0}||\pi_\theta) \approx \frac{\sum_{i=1}^{N} \sum_{t=0}^{T_i-1}\nabla_\theta D_{\text{KL}}(\pi_{\theta_0}(\cdot|s_t)||\pi_{\theta}(\cdot|s_t))}{\sum_{i=1}^{N}{T_i}}
\end{equation}

\xhdr{Additional considerations} Generating counterfactual explanation policies using the above gradients poses a significant practical challenge -- the KL term conservatively restricts the policy to the neighborhood of the original policy $\pi_{\theta_0}$ hindering the actual aim of reaching $\rtarget$. 

Intuitively, we can tune the regularization hyperparameter $k$ on the KL-weight \textit{w.r.t.} the $\rtarget$ at the start of the optimization. Nevertheless, we update the policy pivot $\pi_{\theta_0}$ iteratively to $[\pi_{\theta_1}, \pi_{\theta_2}, \dots]$ after every fixed $m$ number of gradient updates of $\pi_\theta$ to ensure it achieves the desired $\rtarget$ return. Finally, at the $i^{\text{th}}$ iteration, after $m$ steps of gradient updates of the objective:
\begin{equation}
    \label{eqn:cf_iterative_defn}
    \pi_{\theta_\text{cf}} = \argmin_{\theta} (|\Jpitheta - \rtarget| + k \cdot D_{\text{KL}}(\pi_{\theta_{i-1}} || \pi_\theta)),
\end{equation}
we change the pivot policy from $\pi_{\theta_{i-1}}$ to $\pi_{\theta_{i}}$ using $\pi_\theta$ at that step. The above process is continued till $|\Jpitheta{-}\rtarget|$ is less than a certain threshold $\delta$. In Algorithm~\ref{alg:pi_cf_estimation}, we describe the complete step-by-step algorithm of \method, our proposed framework.

\begin{algorithm}[h]
\footnotesize
\caption{\method: Counterfactual Explanation Policy Optimization}
\label{alg:pi_cf_estimation}
\SetKwInOut{Input}{Input}
\SetKwInOut{Output}{Output}

\textbf{Given:}~~$\pi_{\theta_{0}}, \rtarget, \delta, k, m, N, \eta $\\

$\pi_{\theta'} \xleftarrow[]{} \pi_{\theta_0}.\text{copy}()$ \tcp{Create copy of original policy, which will be updated}
\For{i = $0, 1, \dots$}{
    \For{j = $0, 1, \dots, m - 1$}{
        $\{\tau\}_{i=1}^{N} \sim \pi_{\theta'}$ \tcp{Sample $N$ episodes using policy under updation}

        $J_{\pi_{\theta'}} \xleftarrow[]{} \text{estimatePerformance}(\{\tau\}_{i=1}^{N})$ \tcp{using Equation~\ref{eqn:jpitheta_approx}}
        \tcc{Check if stopping criterion is met}
        \If{$|J_{\pi_{\theta'}} - \rtarget| < \delta$}{
            $\pi_{\theta_\text{cf}} \xleftarrow[]{} \pi_{\theta'}$
            
            \text{Exit both the loops.}
        }

        $
        \nabla_{\theta}L_{\text{cf},\theta} \xleftarrow[]{} \nabla_{\theta}L_{\text{ret}, \theta}\rvert_{\theta=\theta'} + k \cdot \nabla_{\theta} L_{\text{KL}}(\pi_{\theta_i} || \pi_\theta)\rvert_{\theta=\theta'}
        $ \tcp{Compute the objective gradient using Equations~\ref{eqn:polgrad_approx} and ~\ref{eqn:grad_dkl_approx}}
        $\theta' \xleftarrow[]{} \theta' - \eta \cdot \nabla_{\theta}{L_{\text{cf}, \theta}}\rvert_{\theta=\theta'}$ \tcp{Update parameters of $\pi_{\theta'}$}
    }
    
    $\pi_{\theta_{i+1}} \xleftarrow{} \pi_{\theta'}$ \tcp{Update the KL-pivot policy to $\pi_{\theta'}$}
}
\textbf{return}~~counterfactual policy $\pi_{\theta_\text{cf}}$
\end{algorithm}

\subsection{Connection between Counterfactual Explanation Policies and Trust Region Methods}
\label{sec:trpo_connection}
Trust region-based RL policy optimization methods like TRPO~\citep{trpo}, ACKTR~\citep{acktr} and PPO~\citep{ppo} have been foundational in the unprecedented success of deep RL~\citep{dota2, gpt4}. The primary aspect behind their optimization involves iteratively updating the policy within a \textit{trust} region, leading to a monotonic improvement in the policy's performance. Formally, the objective is defined as:
\begin{equation}
    \label{eqn:tr_1}
    \pi_{\theta_{i+1}} =  \argmax_{\theta}~~ J_{\pi_\theta}\text{~~such that~~} D_{\text{KL}}(\pi_{\theta_i} || \pi_{\theta}) \leq \delta
\end{equation}
The above objective can be written in its Lagrangian form using penalty instead of the constraint as: 
\begin{equation}
    \label{eqn:tr_2}
    \begin{aligned}
    \pi_{\theta_{i+1}} =  {} & \argmax_{\theta}~~ J_{\pi_\theta} - \lambda\cdot D_{\text{KL}}(\pi_{\theta_i} || \pi_{\theta}), \\
    {} & \text{where $\lambda$ is treated as a hyper-parameter.}
    \end{aligned}
\end{equation}

Next, we show the equivalence between counterfactual explanation policy and the policy obtained using a trust region-based policy gradient method.
\begin{theorem}(Equivalence)
\label{thm:equivalence}
Let $\rtarget$ be the maximum possible return in the MDP under consideration. Then, for any $L_{\text{ret}}$ estimated using $\ell_1$-norm, the generated counterfactual explanation policy through iterative KL-pivoting is equivalent to optimizing the policy for best return using a trust region-based policy gradient method.
\end{theorem}
\begin{proof}
For a given $\rtarget$, let us assume that the desired return from a policy is equal to the maximum possible return, \ie $\rtarget = R_{\text{max}}$, and the return penalty calculated using $\ell_1$-norm, we rewrite the counterfactual generation objective using KL-pivoting~ Eqn.~\ref{eqn:cf_iterative_defn} as:
\begin{equation}
\label{eqn:proof_1}
    \pi_{\theta_{\text{cf}}} = \argmin_{\theta} (|J_{\pi_\theta} - R_{\text{max}}| + k \cdot D_{\text{kl}}(\pi_{\theta_i} || \pi_\theta)) 
\end{equation}
We have $J_{\pi_\theta}=\mathbb{E}_{\pi_\theta}[\sum_{t=0}^{T-1}\gamma^t r(s_t, a_t, s_{t+1})| s_0 {\sim} d_0, a_t \sim \pi_\theta(\cdot| s_t))]$ as defined in Eq.~\ref{eqn:j_pi}. Now, let $R{=}\sum_{t=0}^{T-1}\gamma^t r(s_t, \\a_t, s_{t+1})$ denote the discounted return for the episode $(s_0, a_0, s_1, a_1, \dots, s_T)$, then $J_{\pi_\theta}=\mathbb{E}_{\pi_\theta}[R]$. Further,
\[ 
    R \leq R_\text{max} \implies J_{\pi_\theta} = \mathbb{E}_{\pi_\theta}[R] \leq \mathbb{E}_{\pi_\theta}[R_\text{max}] = R_\text{max}
\]
Hence, $J_{\pi_\theta} - R_\text{max} \leq 0$, or $R_\text{max} - J_{\pi_\theta} \geq 0$, which allows us to write $|J_{\pi_\theta} - R_\text{max}|$ simply as $(R_\text{max} - J_{\pi_\theta})$. Rewriting Eqn.~\ref{eqn:proof_1}, we get:
\begin{equation}
\label{eqn:proof_2}
\begin{aligned}
\pi_{\theta_{\text{cf}}}
= {} & \argmin_{\theta}~~ R_\text{max} - J_{\pi_\theta} + k \cdot D_{\text{KL}}(\pi_{\theta_i} || \pi_\theta) \\
= {} & \argmin_{\theta}~~ - J_{\pi_\theta} + k \cdot D_{\text{KL}}(\pi_{\theta_i} || \pi_\theta) \\
= {} & \argmax_{\theta}~~ J_{\pi_\theta} - k \cdot D_{\text{KL}}(\pi_{\theta_i} || \pi_\theta)
\end{aligned}
\end{equation}
Comparing the above equations with Eqn.~\ref{eqn:tr_2} and treating the hyper-parameters $k$ and $\lambda$ interchangeably, we find that both these objectives of policy updation are the same, which completes our proof.
\end{proof}
Consequently, the trust region-based policy gradient methods can be interpreted as finding a \textit{counterfactual explanation policy} of the original policy, which achieves the maximum possible return (\ie the best performance). This also opens up possibility for using sophisticated ways to choose distance regulation parameter $k$ similar to $\lambda$ following ~\cite{cpi}.
\section{Experiments and Results}
\label{sec:expts}
Here, we study the counterfactual explanations of policies trained using policy gradient methods, particularly the actor-critic methods~\citep{sutton_book}, as they converge more faithfully compared to vanilla policy gradients and also maintain the simplicity required to analyze them in a contrastive fashion rigorously.

\xhdr{Setup} We employ the widely used actor-critic algorithm of Advantage Actor-Critic (A2C)~\citep{a2c_1}, and only in the case of complex environments, we shift to Proximal Policy Optimization (PPO)~\citep{ppo} for training the original policies in our empirical analysis. We use the standard implementations provided in stable-baselines library~\footnote{Documentation for the library: \url{https://stable-baselines3.readthedocs.io/en/master/}}~\citep{stable-baselines3} for RL training using the above-mentioned algorithms. We conduct our analysis in five OpenAI gym environments~\citep{openaigym}, \textit{viz.} i) \textit{CartPole-v1}, ii) \textit{Acrobot-v1}, iii) \textit{Pendulum-v1}, iv) \textit{LunarLander-v2} and v) \textit{BipedalWalker-v3}~\footnote{Additional environment details can be found at \url{https://gymnasium.farama.org/environments/}}. 

\xhdr{Implementation details}
We set the threshold return values of $\delta$ for the five environments to $\{10, 2.5, 37.5, 5, 10\}$ according to the order of magnitude of returns in their respective environment,  the number of KL-pivoting policy iterations $m$ as 10 for all environments (except BipedalWalker, where $m=5$), and the number of trajectory rollouts $N$ to 10 for all four environments except BipedalWalker, where $N=2$. Further, we chose the distance regularizer parameter $k$ values as $\{10, 1, 10^5, 10, 1\}$ for the five environments in the same order as discussed above. For computation purposes, we use a single NVIDIA A100 (40GB) GPU. We share the code for generating counterfactual explanation policies using \method in the supplementary material.
\begin{table*}[ht]
\fontsize{7.5pt}{7.5pt}\selectfont
\centering
\setlength{\tabcolsep}{0.5pt}
\renewcommand{\arraystretch}{1.3}
\caption{\textbf{\method Optimization on Control Environments.} Counterfactual explanations for a given policy $\pi_0$ with performance $J_{\pi_0}$ and three target returns. Shown are the number of outer policy updates ($n_{\pi}$) for generating counterfactuals and the resulting performance of the counterfactual policy $\pi_{\theta_{cf}}$. \method faithfully estimates the counterfactuals across all the $\pi_0$-$\rtarget$ pairs and diverse environments.
}
\label{tab:control_env_results}
\begin{tabular}{lccccccccc}
\toprule
\multicolumn{10}{c}{CartPole-v1} \\
\cmidrule{1-10}
$J_{\pi_{\theta_o}}$ & \multicolumn{3}{c}{235.6} & \multicolumn{3}{c}{368.2} & \multicolumn{3}{c}{500.0} \\ 
$R_\text{target}$ & 50 & 250 & 450 & 50 & 250 & 450 & 50 & 250 & 450 \\
$n_{\pi}$ & 278.0\std{22.9} & 10.0\std{6.5} & 303.3\std{67.1} & 108.0\std{16.6} & 1340.0\std{278.8} & 68.7\std{5.0} & 833.3\std{56.3} & 721.3\std{14.7} & 557.3\std{10.0} \\
\rowcolor{gray!30}
$J_{\pi_{\theta_\text{cf}}}$ & 48.6\std{4.9} & 245.0\std{4.4} & 450.6\std{7.5} & 48.2\std{3.4} & 245.5\std{1.9} & 453.5\std{0.5} & 56.4\std{2.4} & 246.0\std{4.7} & 452.1\std{1.0} \\
\midrule
\multicolumn{10}{c}{Acrobot-v1} \\
\cmidrule{1-10}
$J_{\pi_{\theta_o}}$ & \multicolumn{3}{c}{-146.7} & \multicolumn{3}{c}{-89.0} & \multicolumn{3}{c}{-84.3} \\
$R_\text{target}$ & -120 & -100 & -80 & -120 & -100 & -80 & -120 & -100 & -80 \\
$n_{\pi}$ & 44.0\std{18.8} & 31.3\std{1.9} & 18.0\std{12.8} & 62.0\std{34.8} & 38.0\std{37.3} & 6.0\std{2.8} & 117.3\std{95.3} & 26.0\std{36.8} & 8.7\std{5.7} \\
\rowcolor{gray!30}
$J_{\pi_{\theta_\text{cf}}}$ & -119.5\std{1.3} & -100.5\std{2.1} & -80.6\std{1.6} & -119.1\std{0.5} & -99.9\std{0.2} & -80.6\std{0.5} & -120.1\std{1.5} & -99.8\std{1.8} & -81.6\std{0.3} \\
\midrule
\multicolumn{10}{c}{Pendulum-v1} \\
\cmidrule{1-10}
$J_{\pi_{\theta_o}}$ & \multicolumn{3}{c}{-853.5} & \multicolumn{3}{c}{-792.6} & \multicolumn{3}{c}{-568.0} \\
$R_\text{target}$ & -1000 & -750 & -500 & -1000 & -750 & -500 & -1000 & -750 & -500 \\
$n_{\pi}$ & 1.3\std{0.9} & 15.3\std{14.8} & 174.7\std{38.0} & 6.0\std{3.3} & 10.0\std{11.4} & 198.7\std{157.4} & 113.3\std{139.1} & 22.0\std{25.7} & 6.0\std{3.1} \\
\rowcolor{gray!30}
$J_{\pi_{\theta_\text{cf}}}$ & -996.9\std{19.6} & ~-773.5\std{6.4} & ~-507.9\std{19.7} & ~-992.1\std{25.9} & -777.1\std{6.2} & -507.3\std{7.1} & -997.6\std{15.1} & -764.9\std{19.8} & ~-501.0\std{8.8} \\
\bottomrule
\end{tabular}
\end{table*}
\begin{figure*}[h]
    \begin{flushleft}
        \footnotesize
        \hspace{0.4cm}\rotatebox{90}{\hspace{-6.1cm}{Scenario 1}\hspace{0.58cm}{Scenario 2}\hspace{0.8cm}{Scenario 3}} \hspace{1.4cm}$\pi_{0}$\hspace{2.4cm}$\text{R}_{\text{target}}=100$\hspace{1.3cm}$\text{R}_{\text{target}}=150$\hspace{1.6cm}$\text{R}_{\text{target}}=0$\hspace{1.6cm}$\text{R}_{\text{target}}=-50$
    \end{flushleft}
    \centering
    \includegraphics[width=0.9\textwidth]{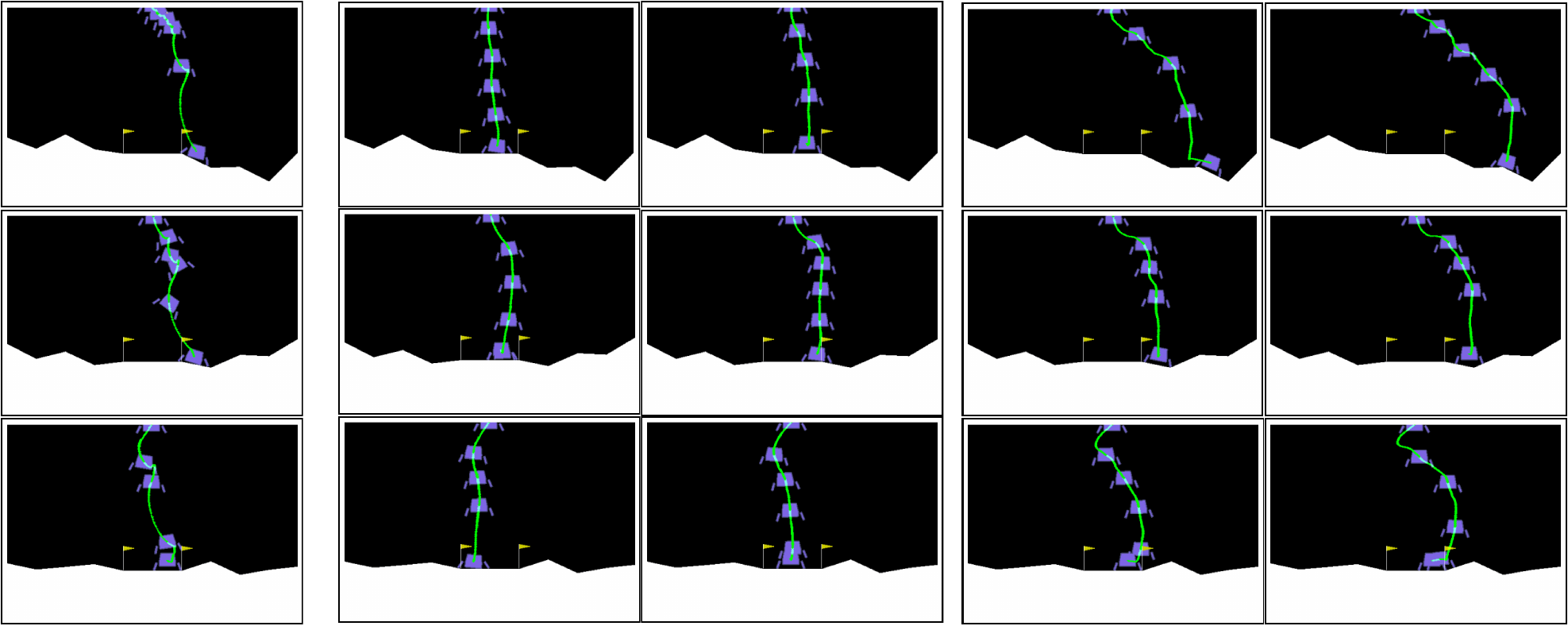}
    \caption{\textbf{\method on LunarLander-v2.} We showcase rollouts of different policies trained to land the rover by sampling frames at every 5$^{\text{th}}$ step and tracking the rover's center in \textcolor{green}{green}. Slight modifications to the original policy can improve (columns 2-3) or worsen (columns 4-5) its overall returns.}
    \label{fig:lunarlander_qualitative}
\end{figure*}
\subsection{Results on \method Optimization}
\looseness=-1
We conduct experiments to understand and verify the optimization of our proposed \method framework. In doing so, we present the results of \method optimization on A2C-trained agents of \textit{CartPole-v1}, \textit{Acrobot-v1}, and \textit{Pendulum-v1} environments. We sample three distinct policy checkpoints from the A2C training of each RL environment and generate counterfactual policies using \method for these checkpoints and three different target returns $\rtarget$ (chosen with respect to the RL environment). We train each tuple of the RL environment, A2C checkpoint, and target return for three different seeds to reduce variance arising from stochasticity. Our results in Table~\ref{tab:control_env_results} show that \method faithfully achieves target returns for diverse starting checkpoints across all environments, \ie \method converges to the target return value generating $\pi_{\text{cf}}$ that obtains return very close to the given $\rtarget$. Further, we find an intuitive trend in the number of outer policy (KL-pivoting) updates, where we 
observe a lesser number of outer policy updates when the $\rtarget$ value is closer to the performance of the original policy $J_{\pi_{\theta_0}}$. Our results on these standard control environments form the basis for further investigation of more complex environments having a discrete or a continuous action space, which we explore in the next section.

\subsection{Contrastive Insights into RL policies}
Next, we present our analysis on generating counterfactual explanations using more complex environments.

\looseness=-1
\xhdr{1) Lunar Lander} We train an RL agent on \textit{LunarLander-v2 }using A2C, and intervene the training to retrieve the policy $\pi_0$ for our contrastive analysis. The original policy, on average, achieves a return of 50 (refer Figure~\ref{fig:lunarlander_quantitative}). We then generate its counterfactuals for target returns of $\{100, 150\}$ to understand how the policy can be improved further and also for target returns of $\{0, -50\}$ to understand how the policy can become worse. We present landing scenarios of the policy $\pi_0$ on three different surfaces in Figure~\ref{fig:lunarlander_qualitative} and their respective contrastive explanations for improvement and deterioration. Our results for $\pi_0$ rollouts show interesting agent characteristics like ``slow start'', ``a quick fall after covering half the distance,'' and ``landing near the right flag''. We find an improved version of the original policy in the generated counterfactuals with $\rtarget=100$ (Fig.~\ref{fig:lunarlander_qualitative}; column 2), where we observe that a uniform descent and shifting the landing slightly to the left between the two flags can improve the original policy $\pi_0$. Further, the counterfactual with  $\rtarget=150$ (Fig.~\ref{fig:lunarlander_qualitative}; column 3) shows uniform descent with decelerated landing can lead to even higher gains. In contrast, $\rtarget=0$ (Fig.~\ref{fig:lunarlander_qualitative}; column 4) shows how by starting very fast and making \textit{free fall} before landing outside the space between flags, $\pi_0$ can go worse. Similarly, $\rtarget=-50$ (Fig.~\ref{fig:lunarlander_qualitative}; column 5) show how the policy could worsen/collapse by making the agent land further right. Notably, the generated counterfactuals for targeted deterioration of performance using \method can be interpreted as a robust way to unlearn~\citep{unlearning, unlearning_survey} RL skills as it might be required to forget certain aspects of learning(\eg in Fig.~\ref{fig:lunarlander_qualitative}, the slow start of $\pi_0$).

\xhdr{2) Bipedal Walker} For the BipedalWalker-v3 environment, we train a PPO agent that demonstrates early success in achieving walking behavior. We conduct experiments to analyze the trained Bipedal agent contrastively by estimating counterfactuals at target returns of 50 and 150 (refer Fig.~\ref{fig:bipedal_convergence}). We demonstrate the qualitative results in Figure~\ref{fig:bipedal_qualitative}, where we find that the given policy $\pi_0$ has a peculiar walk, kneeling on the right leg and taking stride with the left one. In contrast, our generated counterfactual policy using $\rtarget=150$ shows improved (upright and faster) walk of the Bipedal agent. Further, when we reduce the target return to $\rtarget=50$, we observe that the kneeling gets intensified and the agent starts to drag itself to the finishing line, making the agent to walk slow and also fall (as shown in Scenario 1).

Across both environments, we observe that \method generates counterfactual policies that look similar to the original policy, keeping the essential characteristics of the

\begin{figure}[H]
    \centering
    \begin{subfigure}[b]{0.36\textwidth}
        \centering
        \includegraphics[width=\textwidth]{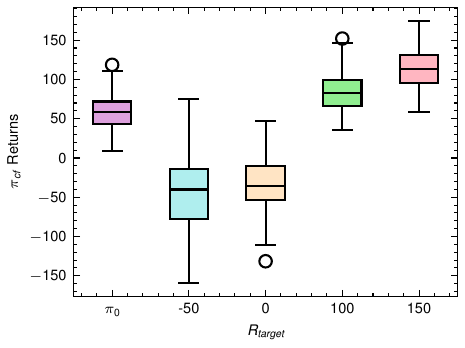}
        \caption{LunarLander returns for $\pi_0$ and its counterfactuals. Results aggregated over 100 surfaces.}
        \label{fig:lunarlander_quantitative}
    \end{subfigure}
    \hfill
    \begin{subfigure}[b]{0.35\textwidth}
        \centering
        \includegraphics[width=\textwidth]{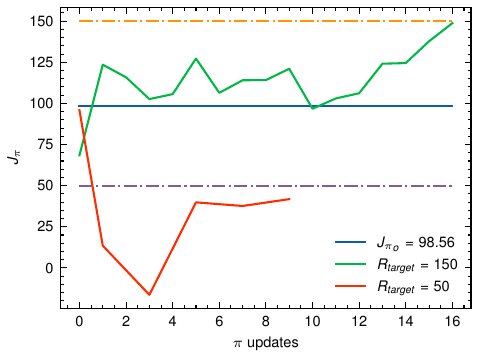}
        \caption{BipedalWalker counterfactual convergence on target returns of 150 and 50.}
        \label{fig:bipedal_convergence}
    \end{subfigure}
    \caption{\textbf{\method Quantitative Analysis on Box-2D Environments.} (a) shows the box plot for LunarLander returns and (b) depicts the convergence of our optimization algorithm on BipedalWalker. Once again, we find that \method successfully reaches the given target returns even with increased environment complexity.}
    \label{fig:quantitative_plots}
\end{figure}
policy the same. This enables us to easily contrast between different versions of the same policies.
\begin{figure*}
    \begin{flushleft}
        \hspace{3.1cm}{Scenario 1}\hspace{5.2cm}{Scenario 2}
    \end{flushleft}
  \begin{minipage}[b]{0.49\textwidth}
    \begin{subfigure}[b]{\textwidth}
    \includegraphics[width=\textwidth]{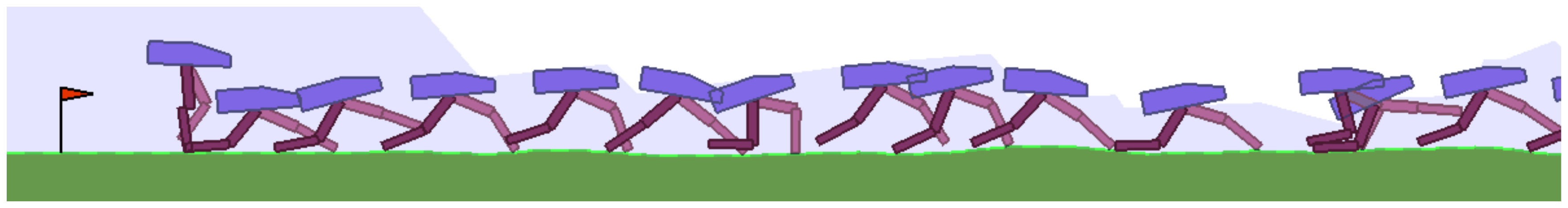}
    \caption{$\pi_0$~($J_{\pi_{\theta_0}}=100$)}
    \end{subfigure}
    \begin{subfigure}[b]{\textwidth}
    \includegraphics[width=\textwidth]{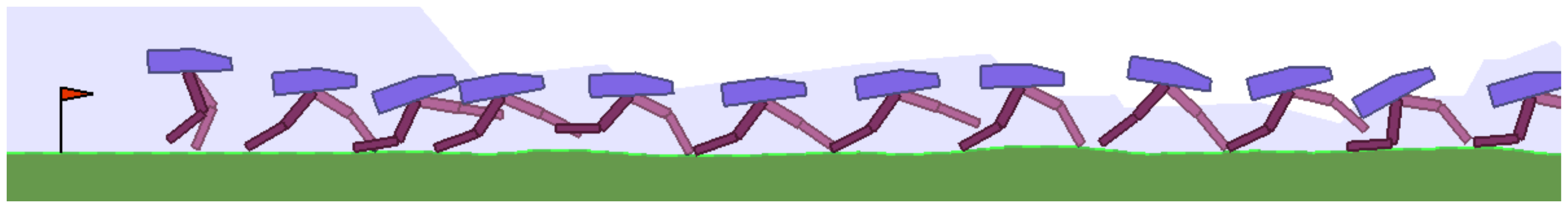}
    \caption{$\rtarget = 150$ (faster and more upright)}
    \end{subfigure}
    \begin{subfigure}[b]{\textwidth}
    \includegraphics[width=\textwidth]{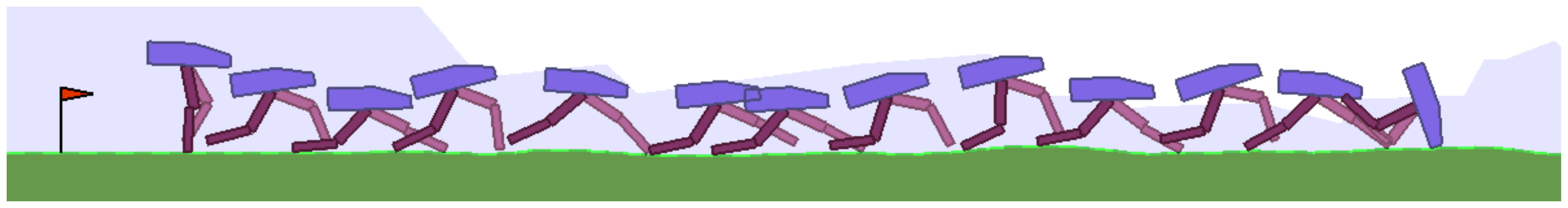}
    \caption{$\rtarget = 50$ (falling before the finish)}
    \end{subfigure}
  \end{minipage}
  \begin{minipage}[b]{0.49\textwidth}
      \begin{subfigure}[b]{\textwidth}
    \includegraphics[width=\textwidth]{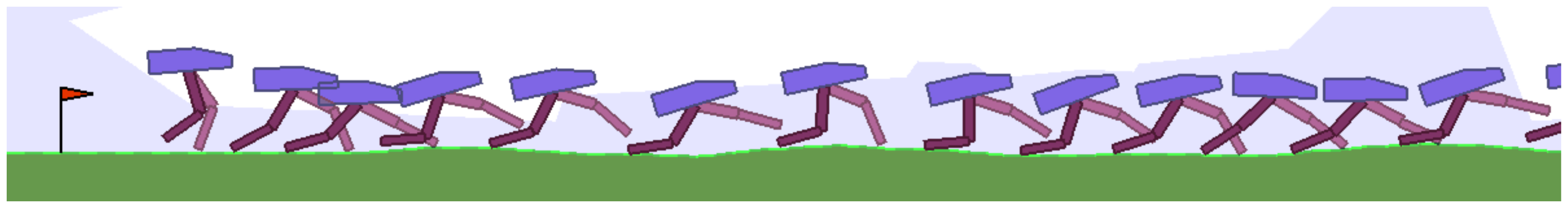}
    \caption{$\pi_0$~($J_{\pi_{\theta_0}}=100$)}
    \end{subfigure}
    \begin{subfigure}[b]{\textwidth}
    \includegraphics[width=\textwidth]{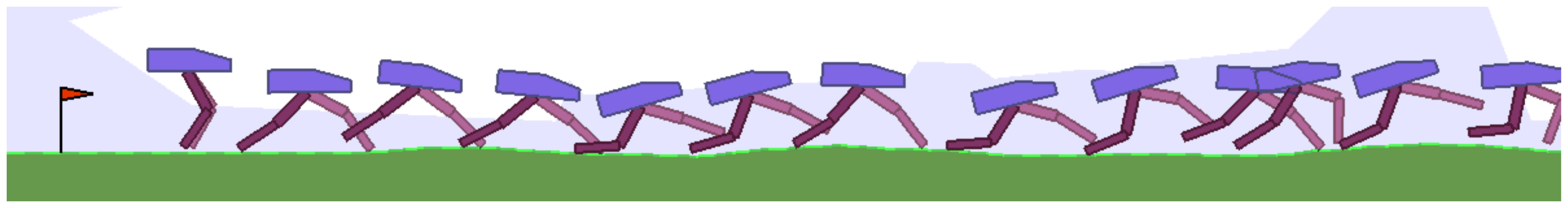}
    \caption{$\rtarget = 150$ (faster and more upright)}
    \end{subfigure}
    \begin{subfigure}[b]{\textwidth}
    \includegraphics[width=\textwidth]{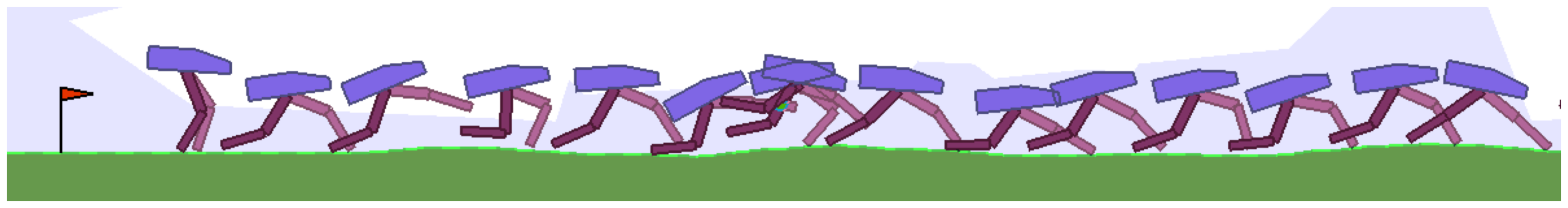}
    \caption{$\rtarget = 50$ (increased dragging on right knee)}
    \end{subfigure}
  \end{minipage}
  \caption{\textbf{\method on BipedalWalker-v3.} Shown are the rollouts of PPO agent trained on BipedalWalker-v3 environment and their respective counterfactuals generated using \method by sampling at every 50$^{\text{th}}$ step. Counterfactuals generated using \method demonstrate contrastive behavior compared to their original policy: i) improved pace and uprightness (Fig.~(b)) of the agent for $\rtarget=150$ and ii) increased kneeling (Fig.~(f)) resulting in worse walking for $\rtarget=50$. Refer to supplementary material for more qualitative results.}
  \label{fig:bipedal_qualitative}
\end{figure*}
\section{Conclusion and Discussion}
\label{sec:conclusion}
In this work, we present a systematic framework for contrastive analysis of reinforcement learning policies. In doing so, we propose \method for generating counterfactual explanations for a given RL policy. We  demonstrate the flexibility of \method across multiple RL benchmarks and find insights into RL policy learning. We carry out a detailed theoretical analysis to connect the counterfactual estimation with eminent trust region optimization methods in RL. Further, results across five OpenAI gym environments show that \method generates faithful counterfactuals for (un)learning skills while keeping close to the original policy. To the best of our knowledge, our work presents the first technique to generate counterfactual explanation policies. Overall, we believe our work paves the way toward a new direction in the methodical contrastive analysis of RL policies and brings more transparency to RL decision-making. In our present work, we mainly used on-policy policy gradient-based RL optimization techniques, which are not very sample efficient, and we assume the stochastic nature of the policies which might not be the case always. It would be interesting to explore the utility of \method in unlearning options (skills) in hierarchical RL~\citep{hrl,hrl_survey}. Also, \method gives a glimpse into return contours in RL which could be further explored for enhanced optimization.

\bibliography{references}
\bibliographystyle{icml2023}



\end{document}